\documentclass[10pt,twocolumn,letterpaper]{article}

\usepackage{cvpr}
\usepackage{times}
\usepackage{epsfig}
\usepackage{graphicx}
\usepackage{amsmath}
\usepackage{amssymb}

\usepackage{array}
\usepackage{subfigure}
\usepackage{amsmath}
\usepackage{amssymb}
\usepackage{amsthm}
\usepackage{mathtools}
\usepackage{algorithm}
\usepackage{algorithmic}

\usepackage{stfloats}
\usepackage{multirow}
\newtheorem{proposition}{Proposition}

\usepackage[pagebackref=true,breaklinks=true,letterpaper=true,colorlinks,bookmarks=false]{hyperref}

\cvprfinalcopy 

\def\cvprPaperID{432} 

\ifcvprfinal\fi
\begin{document}

\title{AdderNet: Do We Really Need Multiplications in Deep Learning?}

\author{
	Hanting Chen$^{1,2}$\thanks{Equal contribution}, Yunhe Wang$^{2*}$, Chunjing Xu$^2$\thanks{Corresponding author}, Boxin Shi$^{3,4}$, Chao Xu$^1$, Qi Tian$^2$, Chang Xu$^5$\\
	\normalsize$^1$ Key Lab of Machine Perception (MOE), CMIC, School of EECS, Peking University, China\\
	\normalsize$^2$ Huawei Noah's Ark Lab, 	$^3$ National Engineering Laboratory for Video Technology, Peking University,\\ \normalsize$^4$ Peng Cheng Laboratory,
	$^5$ School of Computer Science, Faculty of Engineering, The University of Sydney, Australia\\
	\small\texttt{\{htchen, shiboxin\}@pku.edu.cn, xuchao@cis.pku.edu.cn, c.xu@sydney.edu.au}\\
	\small\texttt{\{yunhe.wang, xuchunjing, tian.qi1\}@huawei.com} \\ 
}

\maketitle

\begin{abstract}
   Compared with cheap addition operation, multiplication operation is of much higher computation complexity. The widely-used convolutions in deep neural networks are exactly cross-correlation to measure the similarity between input feature and convolution filters, which involves massive multiplications between float values. In this paper, we present adder networks (AdderNets) to trade these massive multiplications in deep neural networks, especially convolutional neural networks (CNNs), for much cheaper additions to reduce computation costs. In AdderNets, we take the $\ell_1$-norm distance between filters and input feature as the output response. The influence of this new similarity measure on the optimization of neural network have been thoroughly analyzed. To achieve a better performance, we develop a special back-propagation approach for AdderNets by investigating the full-precision gradient. We then propose an adaptive learning rate strategy to enhance the training procedure of AdderNets according to the magnitude of each neuron's gradient. As a result, the proposed AdderNets can achieve 74.9\% Top-1 accuracy 91.7\% Top-5 accuracy using ResNet-50 on the ImageNet dataset without any multiplication in convolution layer. The codes are publicly available at: \url{https://github.com/huaweinoah/AdderNet}.
\end{abstract}

\section{Introduction}
Given the advent of Graphics Processing Units (GPUs), deep convolutional neural networks (CNNs) with billions of floating number multiplications could receive speed-ups and make important strides in a large variety of computer vision tasks, \eg image classification~\cite{VGG,krizhevsky2012imagenet}, object detection~\cite{ren2015faster}, segmentation~\cite{long2015fully}, and human face verification~\cite{wen2016discriminative}. However, the high-power consumption of these high-end GPU cards (\eg 250W+ for GeForce RTX 2080 Ti) has blocked modern deep learning systems from being deployed on mobile devices, \eg smart phone, camera, and watch. Existing GPU cards are far from svelte and cannot be easily mounted on mobile devices. Though the GPU itself only takes up a small part of the card, we need many other hardware for supports, \eg memory chips, power circuitry, voltage regulators and other controller chips. It is therefore necessary to study efficient deep neural networks that can run with affordable computation resources on mobile devices. 

Addition, subtraction, multiplication and division are the four most basic operations in mathematics. It is widely known that multiplication is slower than addition, but most of the computations in deep neural networks are multiplications between float-valued weights and float-valued activations during the forward inference. There are thus many papers on how to trade multiplications for additions, to speed up deep learning. The seminal work~\cite{courbariaux2015binaryconnect} proposed BinaryConnect to force the network weights to be binary (\eg -1 or 1), so that many multiply-accumulate operations can be replaced by simple accumulations.  After that, Hubara~\etal~\cite{hubara2016binarized} proposed BNNs, which binarized not only weights but also activations in convolutional neural networks at run-time. Moreover, Rastegari~\etal~\cite{rastegari2016xnor} introduced scale factors to approximate convolutions using binary operations and outperform~\cite{hubara2016binarized,rastegari2016xnor} by large margins. Zhou~\etal~\cite{zhou2016dorefa} utilized low bit-width gradient to accelerate the training of binarized networks. Cai~\etal~\cite{cai2017deep} proposed an half-wave Gaussian quantizer for forward approximation, which achieved much closer performance to full precision networks.

Though binarizing filters of deep neural networks significantly reduces the computation cost, the original recognition accuracy often cannot be preserved. In addition, the training procedure of binary networks is not stable and usually requests a slower convergence speed with a small learning rate. Convolutions in classical CNNs are actually cross-correlation to measure the similarity of two inputs. Researchers and developers are used to taking convolution as a default operation to extract features from visual data, and introduce various methods to accelerate the convolution, even if there is a risk of sacrificing network capability. But there is hardly no attempt to replace convolution with another more efficient similarity measure that is better to only involve additions. In fact, additions are of much lower computational complexities than multiplications. Thus, we are motivated to investigate the feasibility of replacing multiplications by additions in convolutional neural networks.

\begin{figure*}[t]
	\centering
	\begin{tabular}{cc}
		\includegraphics[width=0.48\linewidth]{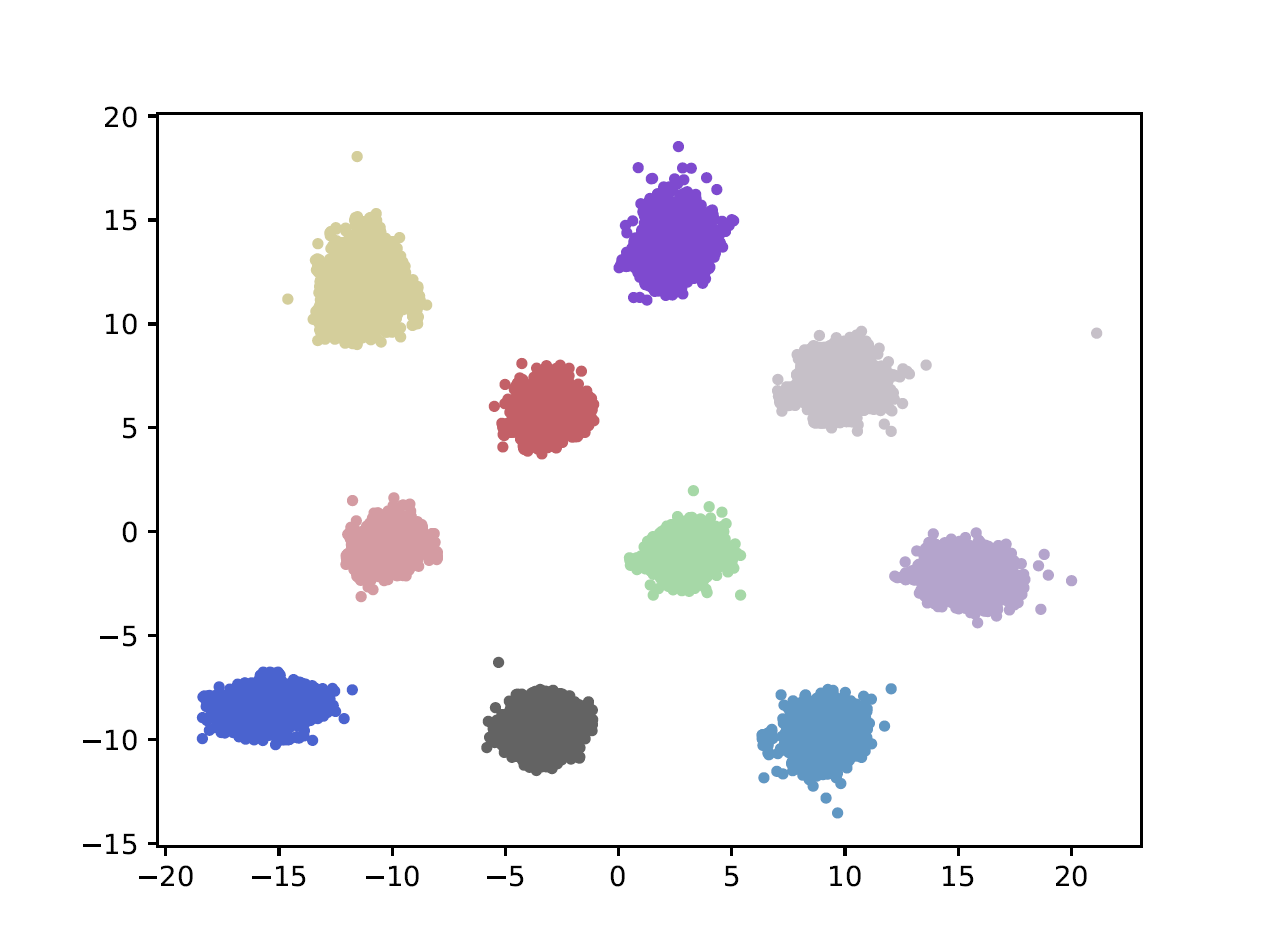} &
		\quad \includegraphics[width=0.48\linewidth]{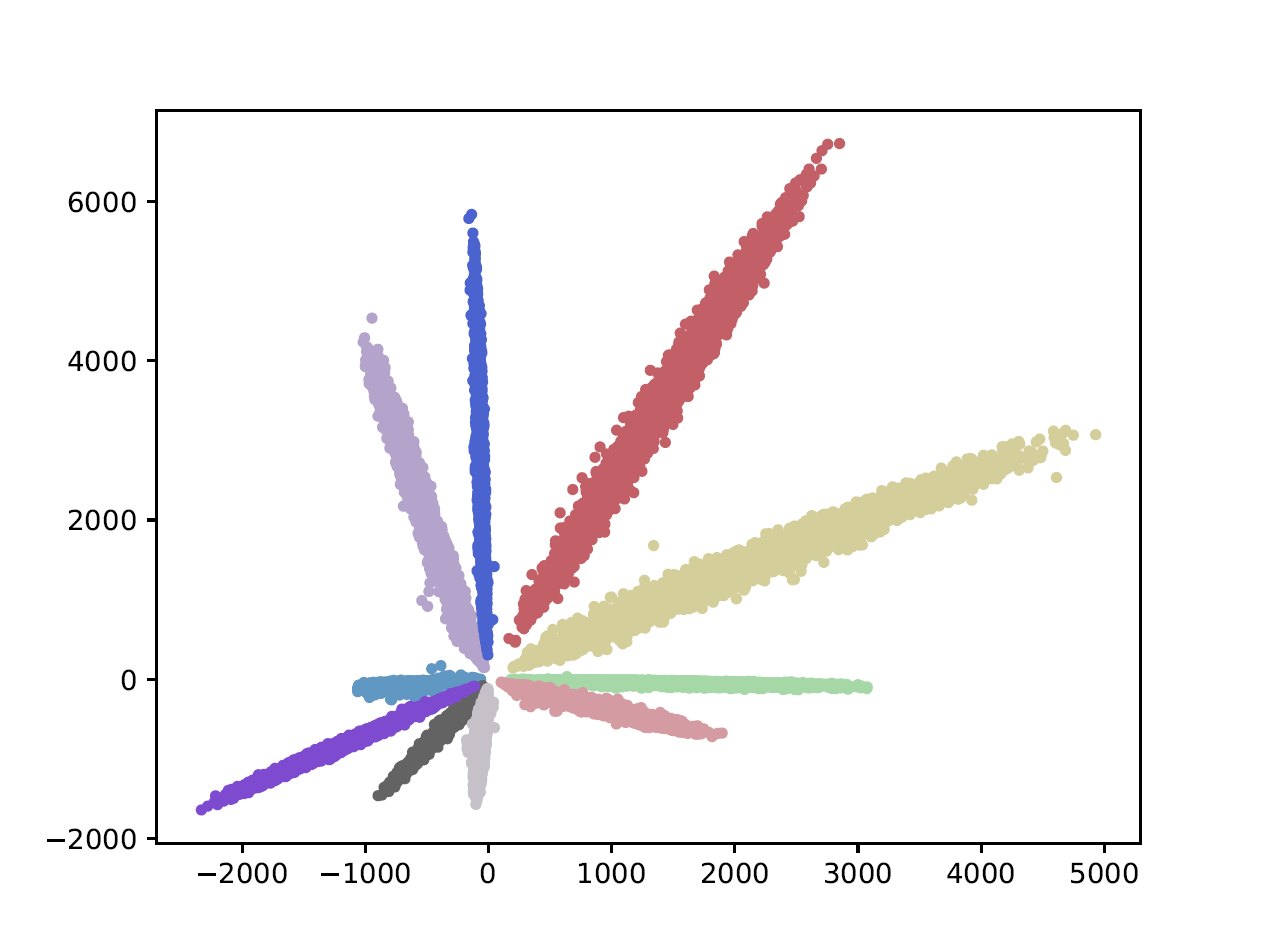} \\
		(a) Visualization of features in AdderNets  &(b)  Visualization of features in CNNs  \\
	\end{tabular}
	\caption{Visualization of features in AdderNets and CNNs. Features of CNNs in different classes are divided by their angles. In contrast, features of AdderNets tend to be clustered towards different class centers, since AdderNets use the $\ell_1$-norm to distinguish different classes. The visualization results suggest that $\ell_1$-distance can served as a similarity measure the distance between the filter and the input feature in deep neural networks}
	\label{Fig:visualfea}
	\vspace{-1em}
\end{figure*}

In this paper, we propose adder networks that maximize the use of addition while abandoning convolution operations. Given a series of small template as ``filters’’ in the neural network, $\ell_1$-distance could be an efficient measure to summarize absolute differences between the input single and the template as shown in Figure~\ref{Fig:visualfea}. Since subtraction can be easily implemented through addition by using its complement code, $\ell_1$-distance could be a hardware-friendly measure that only has additions, and naturally becomes an efficient alternative of the convolution to construct neural networks. An improved back-propagation scheme with regularized gradients is designed to ensure sufficient updates of the templates and a better network convergence. The proposed AdderNets are deployed on several benchmarks, and experimental results demonstrate AdderNets’ advantages in accelerating inference of deep neural networks while preserving comparable recognition accuracy to conventional CNNs.

This paper is organized as follows. Section~\ref{sec:related} investigates related works on network compression. Section~\ref{sec:method} proposes Adder Networks which replace the multiplication in the conventional convolution filters with addition. Section~\ref{sec:experi} evaluates the proposed AdderNets on various benchmark datasets and models and Section~\ref{sec:conclu} concludes this paper.

\section{Related works}\label{sec:related}

To reduce the computational complexity of convolutional neural networks, a number of works have been proposed for eliminating useless calculations. 

Pruning based methods aims to remove redundant weights to compress and accelerate the original network. Denton~\etal~\cite{SVD} decomposed weight matrices of fully-connected layers into simple calculations by exploiting singular value decomposition (SVD). Han~\etal~\cite{han2015deep} proposed discarding subtle weights in pre-trained deep networks to omit their original calculations without affecting the performance. Wang~\etal~\cite{wang2016cnnpack} further converted convolution filters into the DCT frequency domain and eliminated more floating number multiplications. In addition, Hu~\etal~\cite{Trimming} discarded redundant filters with less impacts to directly reduce the computations brought by these filters. Luo~\etal~\cite{luo2017thinet} discarded redundant filters according to the reconstruction error. He~\etal~\cite{he2017channel} utilized a LASSO regression to select important channels by solving least square reconstruction. Zhuang~\etal~\cite{zhuang2018discrimination} introduce additional losses to consider the discriminative power of channels and selected the most discriminative channels for the portable network. 

Instead of directly reducing the computational complexity of a pre-trained heavy neural network, lot of works focused on designing novel blocks or operations to replace the conventional convolution filters. Iandola~\etal~\cite{iandola2016squeezenet} introduced a bottleneck architecture to largely decrease the computation cost of CNNs. Howard~\etal~\cite{howard2017mobilenets} designed MobileNet, which decompose the conventional convolution filters into the point-wise and depth-wise convolution filters with much fewer FLOPs. Zhang~\etal~\cite{zhang2018shufflenet} combined group convolutions~\cite{ResNeXt} and a channel shuffle operation to build efficient neural networks with fewer computations. Hu~\etal~\cite{hu2018squeeze} proposed the squeeze and excitation block, which focuses on the relationship of channels by modeling interdependencies between channels, to improve the performance at slight additional computational cost. Wu~\etal~\cite{wu2018shift} presented a parameter-free ``shift" operation with zero flop and zero parameter to replace conventional filters and largely reduce the computational and storage cost of CNNs. Zhong~\etal~\cite{zhong2018shift} further pushed the shift-based primitives into channel shift, address shift and shortcut shift to reduce the inference time on GPU while keep the performance. Wang~\etal~\cite{Versatile} developed versatile convolution filters to generate more useful features utilizing fewer calculations and parameters.  

Besides eliminating redundant weights or filters in deep convolutional neural networks, Hinton~\etal~\cite{hinton2015distilling} proposed the knowledge distillation (KD)
scheme, which transfer useful information from a heavy teacher network to a portable student network by minimizing the Kullback-Leibler divergence between their outputs. Besides mimic the final outputs of the teacher networks, Romero~\etal~\cite{romero2014fitnets} exploit the hint layer to distill the information in features of the teacher network to the student network. You~\etal~\cite{you2017learning} utilized multiple teachers to guide the training of the student network and achieve better performance. Yim~\etal~\cite{yim2017gift} regarded the relationship between features from two layers in the teacher network as a novel knowledge and introduced the FSP (Flow of Solution Procedure) matrix to transfer this kind of information to the student network.

Nevertheless, the compressed networks using these algorithms still contain massive multiplications, which costs enormous computation resources. As a result, subtractions or additions are of much lower computational complexities when compared with multiplications. However, they have not been widely investigated in deep neural networks, especially in the widely used convolutional networks. Therefore, we propose to minimize the numbers of multiplications in deep neural networks by replacing them with subtractions or additions.

\section{Networks without Multiplication}\label{sec:method}

Consider a filter $F\in \mathbb{R}^{d\times d\times c_{in}\times c_{out}}$ in an intermediate layer of the deep neural network, where kernel size is $d$, input channel is $c_{in}$ and output channel is $c_{out}$. The input feature is defined as $X\in \mathbb{R}^{H\times W \times c_{in}}$, where $H$ and $W$ are the height and width of the feature, respectively. The output feature $Y$ indicates the similarity between the filter and the input feature,  
\begin{equation}
\small
Y(m,n,t) = \sum_{i=0}^{d}\sum_{j=0}^{d}\sum_{k=0}^{c_{in}} S\big(X(m+i,n+j,k), F(i,j,k,t)\big), \label{fcn:conv}
\end{equation} 
where $S(\cdot, \cdot)$ is a pre-defined similarity measure.  If cross-correlation is taken as the metric of distance, \ie $S(x, y) = x \times y$, Eq. (\ref{fcn:conv}) becomes the convolution operation. Eq. (\ref{fcn:conv}) can also implies the calculation of a fully-connected layer when $d=1$. In fact, there are many other metrics to measure the distance between the filter and the input feature. However, most of these metrics involve multiplications, which bring in more computational cost than additions.

\subsection{Adder Networks}

We are therefore interested in deploying distance metrics that maximize the use of additions.  $\ell_1$ distance calculates the sum of the absolute differences of two points’ vector representations, which contains no multiplication. Hence, by calculating $\ell_1$ distance between the filter and the input feature, Eq. (\ref{fcn:conv}) can be reformulated as 
\begin{equation}
\small
Y(m,n,t) = -\sum_{i=0}^{d}\sum_{j=0}^{d}\sum_{k=0}^{c_{in}} \vert X(m+i,n+j,k) - F(i,j,k,t)\vert. \label{fcn:l1}
\end{equation}
Addition is the major operation in $\ell_1$ distance measure, since subtraction can be easily reduced to addition by using complement code. With the help of $\ell_1$ distance, similarity between the filters and features can be efficiently computed. 

Although both $\ell_1$ distance (Eq.  (\ref{fcn:l1}) and cross-correlation in Eq. (\ref{fcn:conv}) can measure the similarity between filters and inputs, there are some differences in their outputs. The output of a convolution filter, as a weighted summation of values in the input feature map, can be positive or negative, but the output of an adder filter is always negative. Hence, we resort to batch normalization for help, and the output of adder layers will be normalized to an appropriate range and all the activation functions used in conventional CNNs can then be used in the proposed AdderNets. Although the batch normalization layer involves multiplications, its computational cost is significantly lower than that of the convolutional layers and can be omitted. Considering a convolutional layer with a filter $F\in \mathbb{R}^{d\times d\times c_{in}\times c_{out}}$, an input $X\in \mathbb{R}^{H\times W \times c_{in}}$ and an output $Y\in \mathbb{R}^{H'\times W' \times c_{out}}$, the computation complexity of convolution and batch normalization is $\mathcal{O}(d^2c_{in}c_{out}HW)$ and $\mathcal{O}(c_{out}H'W')$, respectively. In practice, given an input channel number $c_{in}=512$ and a kernel size $d=3$ in ResNet~\cite{he2016deep}, we have $\frac{d^2c_{in}c_{out}HW}{c_{out}H'W'}\approx 4068$. Since batch normalization layer has been widely used in the state-of-the-art convolutional neural networks, we can simply upgrade these networks into AddNets by replacing their convolutional layers into adder layers to speed up the inference and reduces the energy cost. 

Intuitively, Eq. (\ref{fcn:conv}) has a connection with template matching~\cite{brunelli2009template} in computer vision, which aims to find the parts of an image that match the template. $F$ in Eq. (\ref{fcn:conv}) actually works as a template, and we calculate its matching scores with different regions of the input feature $X$. Since various metrics can be utilized in template matching, it is natural that $\ell_1$ distance can be utilized to replace the cross-correlation in Eq. (\ref{fcn:conv}).

\subsection{Optimization}

Neural networks utilize back-propagation to compute the gradients of filters and stochastic gradient descent to update the parameters. In CNNs, the partial derivative of output features $Y$ with respect to the filters $F$ is calculated as:
\begin{equation}
\frac{\partial Y(m,n,t)}{\partial F(i,j,k,t)} = X(m+i,n+j,k), 
\end{equation} 
where $i\in [m,m+d]$ and $j \in [n,n+d]$. To achieve a better update of the parameters, it is necessary to derive informative gradients for SGD. In AdderNets, the partial derivative of $Y$ with respect to the filters $F$ is:
\begin{equation}
\frac{\partial Y(m,n,t)}{\partial F(i,j,k,t)} = \mbox{sgn} (X(m+i,n+j,k) - F(i,j,k,t)), \label{fcn:l1bp}
\end{equation}  
where $\mbox{sgn}(\cdot)$ denotes the sign function and the value of the gradient can only take +1, 0, or -1.

Considering the derivative of $\ell_2$-norm 
\begin{equation}
\frac{\partial Y(m,n,t)}{\partial F(i,j,k,t)} =  X(m+i,n+j,k) - F(i,j,k,t), \label{fcn:l2bp}
\end{equation} 
Eq. (\ref{fcn:l1bp}) can therefore lead to a signSGD~\cite{bernstein2018signsgd} update of $\ell_2$-norm. However, signSGD almost never takes the direction of steepest descent and the direction only gets worse as dimensionality grows~\cite{bernstein2018convergence}. It is unsuitable to optimize the neural networks of a huge number of parameters using signSGD. Therefore, we propose using Eq. (\ref{fcn:l2bp}) to update the gradients in our AdderNets. The convergence of taking these two kinds of gradient will be further investigated in the supplementary material. Therefore, by utilizing the full-precision gradient, the filters can be updated precisely. 

Besides the gradient of the filters, the gradient of the input features $X$ is also important for the update of parameters. Therefore, we also use the full-precision gradient (Eq. (\ref{fcn:l2bp})) to calculate the gradient of $X$. However, the magnitude of the full-precision gradient may be larger than +1 or -1. Denote the filters and inputs in layer $i$ as $F_i$ and $X_i$. Different from $\frac{\partial Y}{\partial F_i}$ which only affects the gradient of $F_i$ itself, the change of $\frac{\partial Y}{\partial X_i}$ would influence the gradient in not only layer $i$ but also layers before layer $i$ according to the gradient chain rule. If we use the full-precision gradient instead of the sign gradient of $\frac{\partial Y}{\partial X}$ for each layer, the magnitude of the gradient in the layers before this layer would be increased, and the discrepancy brought by using full-precision gradient would be magnified. To this end, we clip the gradient of $X$ to $[-1,1]$ to prevent gradients from exploding. Then the partial derivative of output features $Y$ with respect to the input features $X$ is calculated as:
\begin{equation}
\small
\frac{\partial Y(m,n,t)}{\partial X(m+i,n+j,k)} =  \mbox{HT}(F(i,j,k,t)-X(m+i,n+j,k)) .
\label{fcn:l2bpx}
\end{equation} 
where $\mbox{HT}(\cdot)$ denotes the HardTanh function:
\begin{equation}
\mbox{HT}(x) =    
\left\{
\begin{aligned}
& x \quad&\mbox{if} \quad -1<x<1,\\
& 1   &x>1,\\
& -1   &x<-1.
\end{aligned} 
\right.
\end{equation}

\subsection{Adaptive Learning Rate Scaling}\label{sec:2.3}

In conventional CNNs, assuming that the weights and the input features are independent and identically distributed following normal distribution, the variance of the output can be roughly estimated as:
\begin{equation}
\begin{aligned}
Var[Y_{CNN}]& =  \sum_{i=0}^{d}\sum_{j=0}^{d}\sum_{k=0}^{c_{in}} Var[X\times F] \\
&= d^2c_{in} Var[X]Var[F].
\end{aligned}
\label{varcnn}
\end{equation} 
If variance of the weight is $Var[F]= \frac{1}{d^2c_{in}}$, the variance of output would be consistent with that of the input, which will be beneficial for the information flow in the neural network. In contrast, for AdderNets, the variance of the output can be approximated as:
\begin{equation}
\begin{aligned}
Var[Y_{AdderNet}] &= 
\sum_{i=0}^{d}\sum_{j=0}^{d}\sum_{k=0}^{c_{in}} Var[| X -F|] \\
&= \sqrt{\frac{\pi}{2}}d^2c_{in} (Var[X]+Var[F]),
\end{aligned}
\label{varadd}
\end{equation} 
when $F$ and $X$ follow normal distributions. In practice, the variance of weights $Var[F]$ is usually very small~\cite{glorot2010understanding}, \eg $10^{-3}$ or $10^{-4}$ in an ordinary CNN. Hence, compared with multiplying $ Var[X]$ with a small value in Eq. (\ref{varcnn}), the addition operation in Eq. (\ref{varadd}) tends to bring in a much larger variance of outputs in AdderNets. 

We next proceed to show the influence of this larger variance of outputs on the update of AdderNets. To promote the effectiveness of activation functions, we introduce batch normalization after each adder layer. Given input $x$ over a mini-batch $\mathcal{B} = \left\{x_{1}, \cdots, x_{m}\right\}$, the batch normalization layer can be denoted as:
\begin{equation}
y = \gamma \frac{x-\mu_{\mathcal{B}}}{\sigma_{\mathcal{B}}} + \beta,
\end{equation} 
where $\gamma$ and $\beta$ are parameters to be learned, and $\mu_{\mathcal{B}}= \frac{1}{m}\sum_i x_i$ and $\sigma^2_{\mathcal{B}}= \frac{1}{m}\sum_i (x_i-\mu_{\mathcal{B}})^2$ are the mean and variance over the mini-batch, respectively. The gradient of loss $\ell$ with respect to $x$ is then calculated as:
\begin{equation}
\small
\frac{\partial \ell}{\partial x_i} = \sum_{j=1}^{m} \frac{\gamma}{m^2 \sigma_{\mathcal{B}}} \left\{ \frac{\partial \ell}{\partial y_i} - \frac{\partial \ell}{\partial y_j}[ 1 +  \frac{(x_i-x_j)(x_j- \mu_{\mathcal{B}})}{\sigma_{\mathcal{B}} }] \right\}.
\label{fcn:bn}
\end{equation} 
Given a much larger variance $Var[Y] = \sigma_{\mathcal{B}}$ in  Eq. (\ref{varadd}), the magnitude of the gradient w.r.t $X$ in AdderNets would be much smaller than that in CNNs according to Eq. (\ref{fcn:bn}), and then the magnitude of the gradient w.r.t the filters in AdderNets would be decreased as a result of gradient chain rule.

\begin{algorithm}[t]
	\caption{The feed forward and back propagation of adder neural networks.} 
	\label{alg1} 
	\begin{algorithmic}[1] 
		\REQUIRE 
		An initialized adder network $\mathcal{N}$ and its training set $\mathcal{X}$ and the corresponding labels $\mathcal{Y}$, the global learning rate $\gamma$ and the hyper-parameter $\eta$. 
		\REPEAT
		\STATE Randomly select a batch $\{(\mbox{x},\mbox{y})\}$ from $\mathcal{X}$ and $\mathcal{Y}$;
		\STATE Employ the AdderNet $\mathcal{N}$ on the mini-batch: $\mbox{x} \rightarrow \mathcal{N}(\mbox{x})$;
		\STATE Calculate the full-precision derivative $\frac{\partial Y}{\partial F}$ and $\frac{\partial Y}{\partial X}$ for adder filters using Eq. (\ref{fcn:l2bp}) and Eq. (\ref{fcn:l2bpx});
		\STATE Exploit the chain rule to generate the gradient of parameters in $\mathcal{N}$;
		\STATE Calculate the adapative learning rate $\alpha_l$ for each adder layer according to Eq. (\ref{fcn:lr2}).
		\STATE Update the parameters in $\mathcal{N}$ using stochastic gradient descent.		
		\UNTIL convergence
		\ENSURE A well-trained adder network $\mathcal{N}$ with almost no multiplications.
	\end{algorithmic} 
\end{algorithm}

\begin{table}[h]
	\centering
	\caption{The $\ell_2$-norm of gradient of weight in each layer using different networks at 1st iteration.}
	\begin{tabular}{|c|c|c|c|}
		\hline
		\textbf{Model} & \textbf{Layer 1} & \textbf{Layer 2} & \textbf{Layer 3} \\
		\hline	
		\hline	
		AdderNet & 0.0009 & 0.0012 & 0.0146 \\
		\hline	
		CNN & 0.2261 & 0.2990 & 0.4646 \\
		\hline
	\end{tabular}
	\label{tab:weight}
\end{table}

Table~\ref{tab:weight} reports the $\ell_2$-norm of gradients of filters $\Vert F \Vert _2$ in LeNet-5-BN using CNNs and AdderNets on the MNIST dataset during the 1st iteration. LeNet-5-BN denotes the LeNet-5~\cite{lenet} adding an batch normalization layer after each convolutional layer. As shown in this table, the norms of gradients of filters in AdderNets are much smaller than that in CNNs, which could slow down the update of filters in AdderNets. 

\begin{table*}[t]
	\centering
	\caption{Classification results on the CIFAR-10 and CIFAR-100 datasets.}
	\begin{tabular}{|c|c|c|c|c|c|c|}
		\hline
		\textbf{Model}&	\textbf{Method} & \textbf{\#Mul.} & \textbf{\#Add.} & \textbf{XNOR}  & \textbf{CIFAR-10} & \textbf{CIFAR-100} \\
		\hline	
		\hline	
		&BNN & 0 & 0.65G &  0.65G & 89.80\% & 65.41\% \\
		\cline{2-7}		
		VGG-small&AddNN & 0 & 1.30G &0  & 93.72\% & 72.64\% \\
		\cline{2-7}		
		&CNN &  0.65G& 0.65G& 0 &93.80\% & 72.73\% \\
		\hline
		\hline	
		&BNN & 0 &41.17M  &  41.17M   & 84.87\% & 54.14\% \\
		\cline{2-7}	
		ResNet-20&AddNN & 0 & 82.34M& 0  & 91.84\% & 67.60\% \\
		\cline{2-7}		
		&CNN & 41.17M & 41.17M &  0 &92.25\% & 68.14\% \\
		\hline
		\hline	
		&BNN & 0 &69.12M  &  69.12M   & 86.74\% & 56.21\% \\
		\cline{2-7}	
		ResNet-32&AddNN & 0 & 138.24M& 0  & 93.01\% & 69.02\% \\
		\cline{2-7}		
		&CNN & 69.12M & 69.12M &  0 &93.29\% & 69.74\% \\
		\hline
	\end{tabular}
	\label{tab:cls}
	\vspace{-1.0em}
\end{table*}

A straightforward idea is to directly adopt a larger learning rate for filters in AdderNets. However, it is worth noticing that the norm of gradient differs much in different layers of AdderNets as shown in Table~\ref{tab:weight}, which requests special consideration of filters in different layers. To this end, we propose an adaptive learning rate for different layers in AdderNets. Specifically, the update for each adder layer $l$ is calculated by
\begin{equation}
\Delta F_l = \gamma \times \alpha_l \times \Delta L(F_l),
\label{fcn:lr1}
\end{equation}
where $\gamma$ is a global learning rate of the whole neural network (\eg for adder and BN layers), $\Delta L(F_l)$ is the gradient of the filter in layer $l$ and $\alpha_l$ is its corresponding local learning rate. As filters in AdderNets act subtraction with the inputs, the magnitude of filters and inputs are better to be similar to extract meaningful information from inputs. Because of the batch normalization layer, the magnitudes of inputs in different layers have been normalized, which then suggests a normalization for the magnitudes of filters in different layers. The local learning rate can therefore be defined as:
\begin{equation}
\alpha_l = \frac{\eta\sqrt{k}}{\Vert \Delta L(F_l)\Vert_2},
\label{fcn:lr2}
\end{equation}
where $k$ denotes the number of elements in $F_l$, and  $\eta$ is a hyper-parameter to control the learning rate of adder filters. By using the proposed adaptive learning rate scaling, the adder filters in different layers can be updated with nearly the same step. The training procedure of the proposed AdderNet is summarized in Algorithm~\ref{alg1}.

\section{Experiment}\label{sec:experi}

In this section, we implement experiments to validate the effectiveness of the proposed AdderNets on several benchmark datasets, including MNIST, CIFAR and ImageNet. Ablation study and visualization of features are provided to further investigate the proposed method. The experiments are conducted on NVIDIA Tesla V100 GPU in PyTorch.

\subsection{Experiments on MNIST}\label{sec:clas}

To illustrate the effectiveness of the proposed AdderNets, we first train a LeNet-5-BN~\cite{lenet} on the MNIST dataset. The images are resized to $32\times32$ and are pro-precessed following~\cite{lenet}. The networks are optimized using Nesterov Accelerated Gradient (NAG), and the weight decay and the momentum were set as $5\times10^{-4}$ and 0.9, respectively. We train the networks for 50 epochs using the cosine learning rate decay~\cite{loshchilov2016sgdr} with an initial learning rate 0.1. The batch size is set as 256. For the proposed AdderNets, we replace the convolutional filters in LeNet-5-BN with our adder filters. Note that the fully connected layer can be regarded as a convolutional layer, we also replace the multiplications in the fully connect layers with subtractions. We set the hyper-parameter in Eq. (\ref{fcn:lr2}) to be $\eta=0.1$, which achieves best performance compared with other values from the pool $\left\{1,\frac{1}{2},\frac{1}{5},\frac{1}{10},\frac{1}{20}\right\}$. 

The convolutional neural network achieves a $99.4\%$ accuracy with $\sim$435K multiplications and $\sim$435K additions. By replacing the multiplications in convolution with additions, the proposed AdderNet achieves a 99.4\% accuracy, which is the same as that of CNNs, with $\sim$870K additions and almost no multiplication.
In fact, the theoretical latency of multiplications in CPUs is also larger than that of additions and subtractions. There is an instruction table~\footnote{\url{www.agner.org/optimize/instruction_tables.pdf}} which lists the instruction latencies, throughputs and micro-operation breakdowns for Intel, AMD and VIA CPUs. For example, in VIA Nano 2000 series, the latency of float multiplication and addition is 4 and 2, respectively. The AdderNet using LeNet-5 model will have $\sim$1.7M latency while CNN will have $\sim$2.6M latency in this CPU. In conclusion, the AdderNet can achieve similar accuracy with CNN but have fewer computational cost and latency. Noted that CUDA and cuDNN optimized adder convolutions are not yet available, we do not compare the actual inference time.

\subsection{Experiments on CIFAR}

 We then evaluate our method on the CIFAR dataset, which consist of $32\times32$ pixel RGB color images. Since the binary networks~\cite{zhou2016dorefa} can use the XNOR operations to replace multiplications, we also compare the results of binary neural networks (BNNs). We use the same data augmentation and pro-precessing in He~\etal~\cite{he2016deep} for training and testing. Following Zhou~\etal~\cite{zhou2016dorefa}, the learning rate is set to 0.1 in the beginning and then follows a polynomial learning rate schedule. The models are trained for 400 epochs with a 256 batch size. We follow the general setting in binary networks to set the first and last layers as full-precision convolutional layers. In AdderNets, we use the same setting for a fair comparison. The hyper-parameter $\eta$ is set to 0.1 following the experiments on the MNIST dataset.

The classification results are reported in Table~\ref{tab:cls}. Since computational cost in batch normalization layer, the first layer and the last layer are significantly less than other layers, we omit these layers when counting FLOPs. We first evaluate the VGG-small model~\cite{cai2017deep} in the CIFAR-10 and CIFAR-100 dataset. As a result, the AdderNets achieve nearly the same results (93.72\% in CIFAR-10 and 72.64\% in CIFAR-100) with CNNs (93.80\% in CIFAR-10 and 72.73\% in CIFAR-100) with no multiplication. Although the model size of BNN is much smaller than those of AdderNet and CNN, its accuracies are much lower (89.80\% in CIFAR-10 and 65.41\% in CIFAR-100). We then turn to the widely used ResNet models (ResNet-20 and ResNet-32) to further investigate the performance of different networks. As for the ResNet-20, Tte convolutional neural networks achieve the highest accuracy (\ie 92.25\% in CIFAR-10 and 68.14\% in CIFAR-100) but with a large number of multiplications (41.17M). The proposed AdderNets achieve a 91.84\% accuracy in CIFAR-10 and a 67.60\% accuracy in CIFAR-100 without multiplications, which is comparable with CNNs. In contrast, the BNNs only achieve 84.87\% and 54.14\% accuracies in CIFAR-10 and CIFAR-100. The results in ResNet-32 also suggest that the proposed AdderNets can achieve similar results with conventional CNNs.

\begin{table*}[t]
	\centering
	\caption{Classification results on the ImageNet datasets.}
	\begin{tabular}{|c|c|c|c|c|c|c|}
		\hline
		\textbf{Model} & \textbf{Method} & \textbf{\#Mul.} & \textbf{\#Add.} & \textbf{XNOR} & \textbf{Top-1 Acc.} & \textbf{Top-5 Acc.} \\
		\hline	
		\hline	
		& BNN & 0 & 1.8G &  1.8G &  51.2\% &  73.2\% \\
		\cline{2-7}		
		ResNet-18&AddNN & 0 & 3.6G &0  & 67.0\% & 87.6\% \\
		\cline{2-7}	
		&CNN &  1.8G& 1.8G&  0  & 69.8\% & 89.1\% \\
		\hline
		\hline	
		& BNN & 0 & 3.9G & 3.9G & 55.8\% &  78.4\% \\
		\cline{2-7}	
		ResNet-50&AddNN & 0 & 7.7G &0  & 74.9\% & 91.7\% \\
		\cline{2-7}		
		&CNN &  3.9G& 3.9G&  0  & 76.2\% & 92.9\% \\
		\hline
	\end{tabular}
	\vspace{-1.0em}
	\label{tab:ImageNet}
\end{table*}

\begin{figure*}[t]
	\centering
	\begin{tabular}{cc}
		\includegraphics[width=0.45\linewidth]{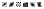} &
		\quad \includegraphics[width=0.45\linewidth]{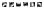} \\
		(a) Visualization of filters of AdderNets  &(b)  Visualization of filters of CNNs  \\
	\end{tabular}
	\caption{Visualization of filters in the first layer of LeNet-5-BN on the MNIST dataset. Both of them can extract useful features for image classification.}
		\vspace{-1.0em}
	\label{Fig:visualfilters}
\end{figure*}

\subsection{Experiments on ImageNet}

 We next conduct experiments on the ImageNet dataset~\cite{krizhevsky2012imagenet}, which consist of $224\times224$ pixel RGB color images. We use ResNet-18 model to evaluate the proposed AdderNets follow the same data augmentation and pro-precessing in He~\etal~\cite{he2016deep}. We train the AdderNets for 150 epochs utilizing the cosine learning rate decay~\cite{loshchilov2016sgdr}. These networks are optimized using Nesterov Accelerated Gradient (NAG), and the weight decay and the momentum are set as $10^{-4}$ and 0.9, respectively. The batch size is set as 256 and the hyper-parameter in AdderNets is the same as that in CIFAR experiments. 

Table~\ref{tab:ImageNet} shows the classification results on the ImageNet dataset by exploiting different nerual networks. The convolutional neural network achieves a 69.8\% top-1 accuracy and an 89.1\% top-5 accuracy in ResNet-18. However, there are 1.8G multiplications in this model, which bring enormous computational complexity. Since the addition operation has smaller computational cost than multiplication, we propose AdderNets to replace the multiplications in CNNs with subtractions. As a result, our AdderNet achieve a 66.8\% top-1 accuracy and an 87.4\% top-5 accuracy in ResNet-18, which demonstrate the adder filters can extract useful information from images. Rastegari~\etal~\cite{rastegari2016xnor} proposed the XNOR-net to replace the multiplications in neural networks with XNOR operations. Although the BNN can achieve high speed-up and compression ratio, it achieves only a 51.2\% top-1 accuracy and a 73.2\% top-5 accuracy in ResNet-18, which is much lower than the proposed AdderNet. We then conduct experiments on a deeper architecture (ResNet-50). The BNN could only achieve a 55.8\% top-1 accuracy and a 78.4\% top-5 accuracy using ResNet-50. In contrast, the proposed AdderNets can achieve a 74.9\% top-1 accuracy and a 91.7\% top-5 accuracy, which is closed to that of CNN (76.2\% top-1 accuracy and 92.9\% top-5 accuracy).

\subsection{Visualization Results}

\textbf{Visualization on features.} The AdderNets utilize the $\ell_1$-distance to measure the relationship between filters and input features instead of cross correlation in CNNs. Therefore, it is important to further investigate the difference of the feature space in AdderNets and CNNs. We train a LeNet++ on the MNIST dataset following~\cite{centerloss}, which has six convolutional layers and a fully-connected layer for extracting powerful 3D features. Numbers of neurons in each convolutional layer are 32, 32, 64, 64, 128, 128, and 2, respectively. For the proposed AdderNets, the last fully connected layers are replaced with the proposed add filters. 

The visualization results are shown in Figure~\ref{Fig:visualfea}. The convolutional neural network calculates the cross correlation between filters and inputs. If filters and inputs are approximately normalized, convolution operation is then equivalent to calculate cosine distance between two vectors. That is probably the reason that features in different classes are divided by their angles in Figure~\ref{Fig:visualfea}. In contrast, AdderNets utilize the $\ell_1$-norm to distinguish different classes. Thus, features tend to be clustered towards different class centers. The visualization results demonstrate that the proposed AdderNets could have the similar discrimination ability to classify images as CNNs.

\textbf{Visualization on filters.} We visualize the filters of the LeNet-5-BN network in Figure~\ref{Fig:visualfilters}. Although the AdderNets and CNNs utilize different distance metrics, filters of the proposed adder networks (see Figure~\ref{Fig:visualfilters} (a)) still share some similar patterns with convolution filters (see Figure~\ref{Fig:visualfilters} (b)). The visualization experiments further demonstrate that the filters of AdderNets can effectively extract useful information from the input images and features. 

\textbf{Visualization on distribution of weights.} We then visualize the distribution of weights for the 3th convolution layer on LeNet-5-BN. As shown in Figure~\ref{Fig:his}, the distribution of weights with AdderNets is close to a Laplace distribution while that with CNNs looks more like a Gaussian distribution. In fact, the prior distribution of $\ell_1$-norm is Laplace distribution~\cite{stigler1986history} and that of $\ell_2$-norm is Gaussian distribution~\cite{rennie2003l2} and the $\ell_2$-norm is exactly same as the cross correlation, which will be analyzed in the supplemental material. 

\begin{figure*}[t]
	\centering
	\begin{tabular}{cc}
		\includegraphics[width=0.46\linewidth]{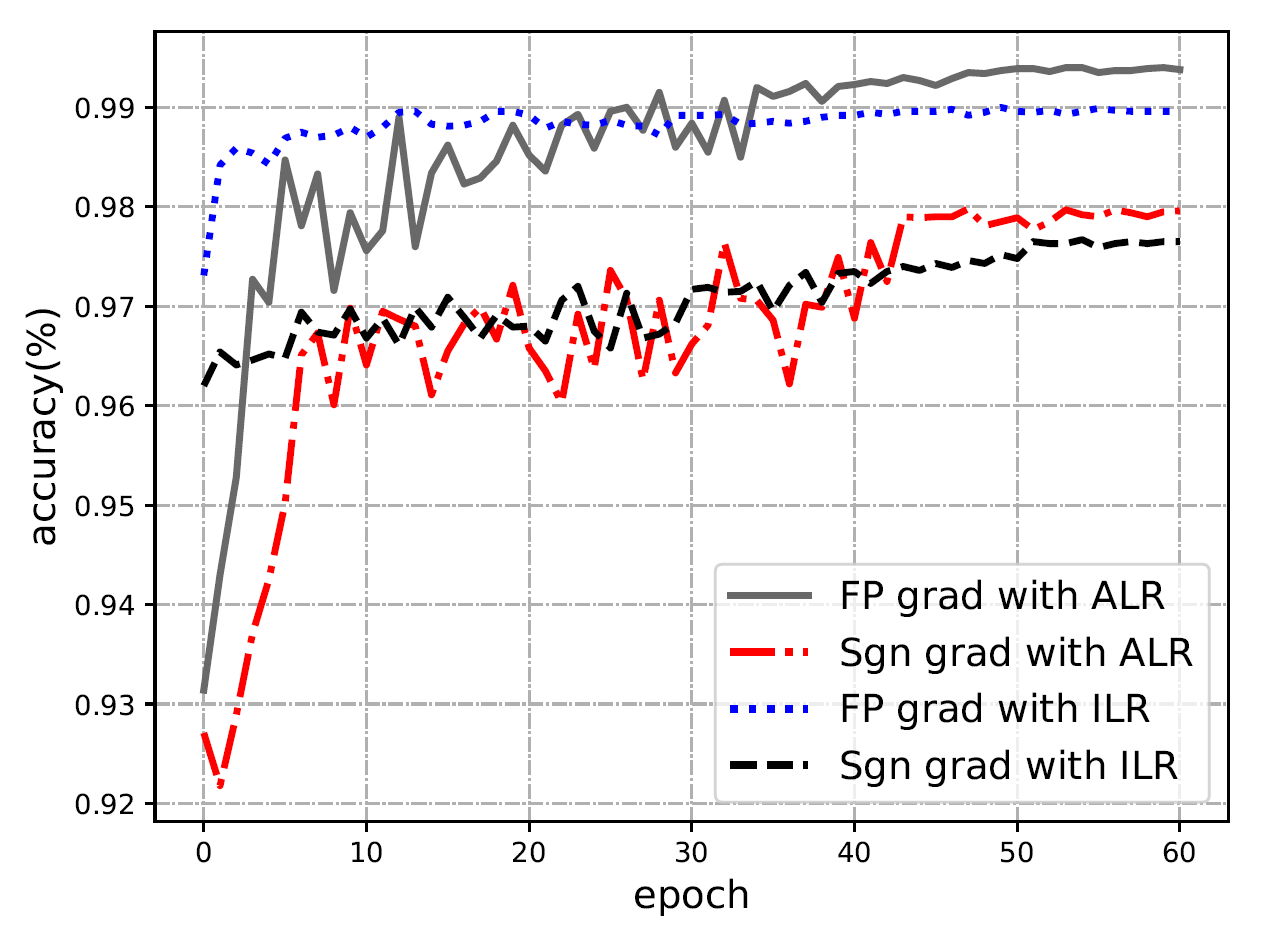} &
		\quad \includegraphics[width=0.46\linewidth]{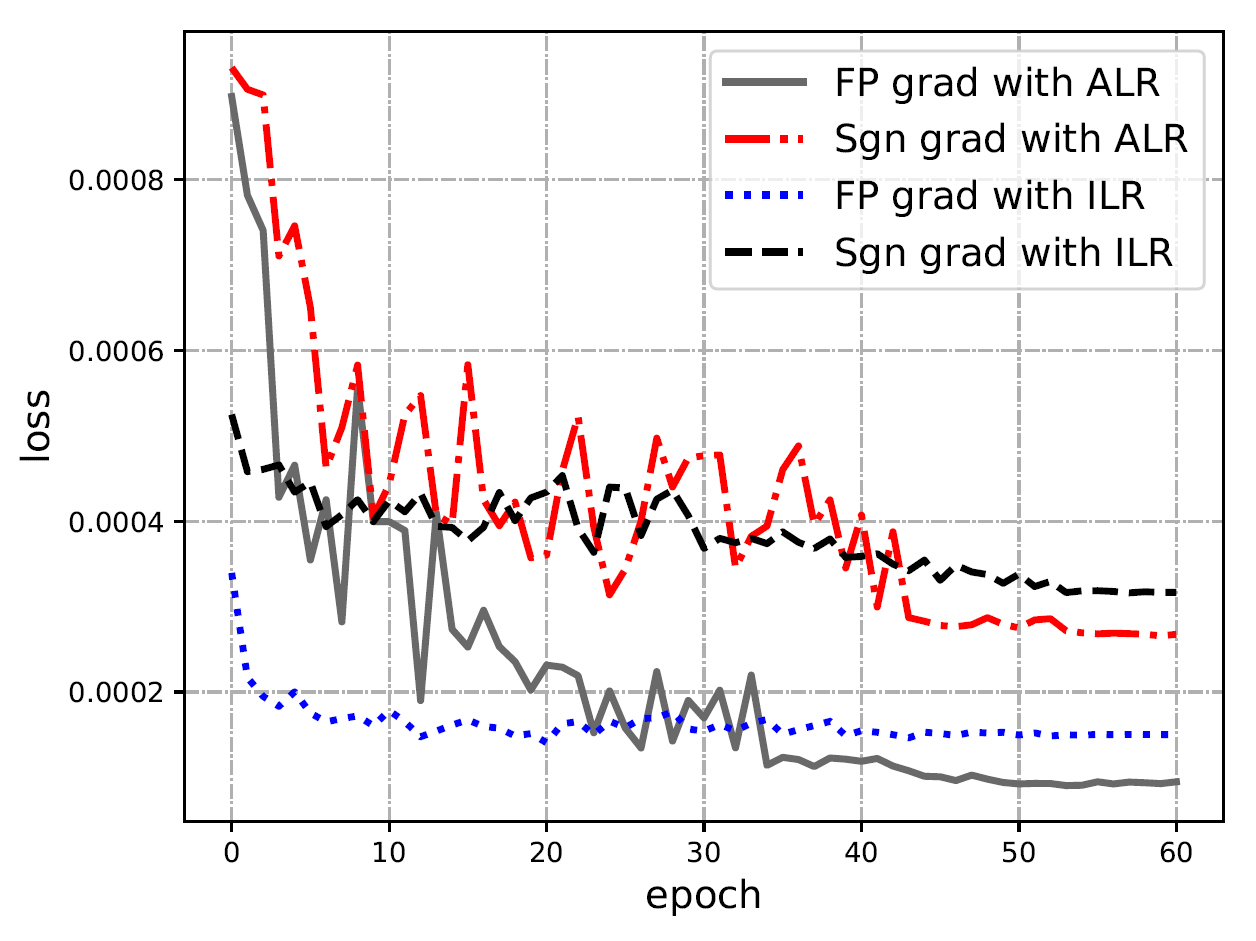} \\
		(a) Accuracy   &(b)  Loss  
	\end{tabular}
	\caption{Learning curve of AdderNets using different optimization schemes. FP and Sgn gradient denotes the full-precision and sign gradient. The proposed adaptive learning rate scaling with full-precision gradient achieves the highest accuracy (99.40\%) with the smallest loss.}
	\label{Fig:abl}
		\vspace{-1.0em}
\end{figure*}
\begin{figure}[h]
	\centering
	\includegraphics[width=1.0\linewidth]{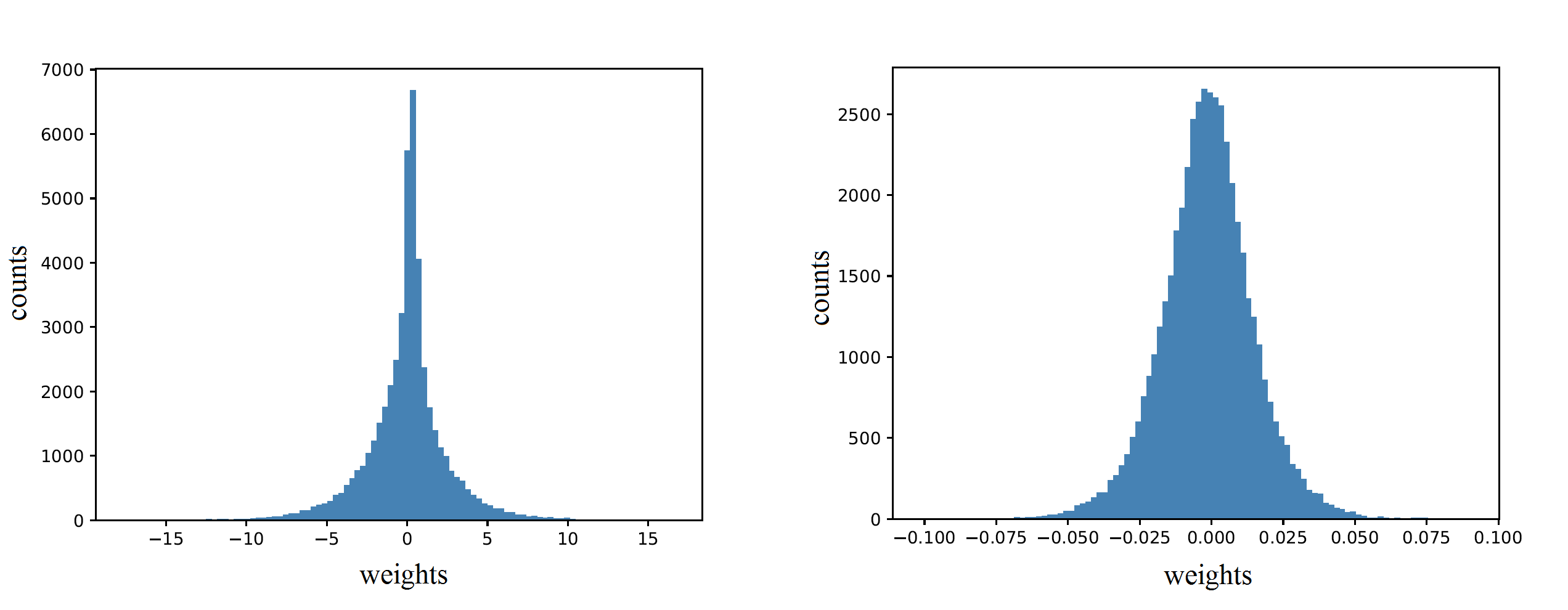}
	\caption{Histograms over the weights with AdderNet (left) and CNN (right). The weights of AdderNets follow Laplace distribution while those of CNNs follow Gaussian distribution. }
	\label{Fig:his}
	\vspace{-1.0em}
\end{figure}

\subsection{Ablation Study}

We propose to use a full-precision gradient to update the filters in our adder filters and design an adaptive learning rate scaling for deal with different layers in AdderNets. It is essential to evaluate the effectiveness of these components. We first train the LeNet-5-BN without changing its learning rate, which results in 54.91\% and 29.26\% accuracies using full-precision gradient and sign gradient, respectively. The networks can be hardly trained since its gradients are very small. Therefore, it is necessary to increase the learning rate of adder filters.

We directly increase the learning rate for filters in AdderNets by 100, which achieves best performance with full-precision gradient compared with other values from the pool $\left\{10,50,100,200,500\right\}$. As shown in Figure~\ref{Fig:abl}, the AdderNets using adaptive learning rate (ALR) and increased learning rate (ILR) achieve 97.99\% and 97.72\% accuracy with sign gradient, which is much lower than the accuracy of CNN (99.40\%). Therefore, we propose the full-precision gradient to precisely update the weights in AdderNets. As a result, the AdderNet with ILR achieves a 98.99\% accuracy using the full-precision gradient. By using the adaptive learning rate (ALR), the AdderNet can achieve a 99.40\% accuracy, which demonstrate the effectiveness of the proposed ALR method. 

\begin{table}[h]
	\centering
	\caption{The impact of parameter $\eta$ using LeNet-5-BN on the MNIST dataset.}
	\begin{tabular}{|c|c|c|c|c|c|}
		\hline
		$\eta$ & 1 & 0.5 & 0.2 & 0.1 & 0.05 \\
		\hline	
		\hline	
		Acc. (\%) & 99.26 &99.30&99.35& 99.40 & 99.32 \\
		\hline
	\end{tabular}
	\label{tab:impact}
\end{table}

\textbf{Impact of parameters.} As discussed above, the proposed adaptive learning rate scaling has a hyper-parameter: $\eta$. We then test its impact on the accuracy of the student network by conducting the experiments on the MNIST dataset. We use LeNet-5-BN as the backbone of AdderNet. Other experimental settings are same as mentioned in Sec.~\ref{sec:clas}. It can be seen from Table~\ref{tab:impact} that the AdderNets trained utilizing the adaptive learning rate scaling achieves the highest accuracy (99.40\%) when $\eta$ = 0.1. Based on the above analysis, we keep the setting of hyper-parameters for the proposed method.

\section{Conclusions}\label{sec:conclu}
The role of classical convolutions used in deep CNNs is to measure the similarity between features and filters, and we are motivated to replace convolutions with more efficient similarity measure. We investigate the feasibility of replacing multiplications by additions in this work.  An AdderNet is explored to effectively use addition to build deep neural networks with low computational costs. This kind of networks calculate the $\ell_1$-norm distance between features and filters. Corresponding optimization method is developed by using regularized full-precision gradients. Experiments conducted on benchmark datasets show that AdderNets can well approximate the performance of CNNs with the same architectures, which could have a huge impact on future hardware design. Visualization results also demonstrate that the adder filters are promising to replace original convolution filters for computer vision tasks. In our future work, we will investigate quantization results for AdderNets to achieve higher speed-up and lower energy comsumption, as well as the generality of AdderNets not only for image classification but also for detection and segmentation tasks.

\section*{Acknowledgement}
We thank anonymous reviewers for their helpful comments. This work is supported by National Natural Science Foundation of China under Grant No. 61876007, 61872012, National Key R\&D Program of China (2019YFF0302902), Beijing Academy of Artificial Intelligence (BAAI), and Australian Research Council under Project DE-180101438.

{\small
\bibliographystyle{ieee_fullname}
\bibliography{ref}

\begin{thebibliography}{10}\itemsep=-1pt

\bibitem{bernstein2018convergence}
Jeremy Bernstein, Kamyar Azizzadenesheli, Yu-Xiang Wang, and Anima Anandkumar.
\newblock Convergence rate of sign stochastic gradient descent for non-convex
  functions.
\newblock 2018.

\bibitem{bernstein2018signsgd}
Jeremy Bernstein, Yu-Xiang Wang, Kamyar Azizzadenesheli, and Anima Anandkumar.
\newblock signsgd: Compressed optimisation for non-convex problems.
\newblock {\em arXiv preprint arXiv:1802.04434}, 2018.

\bibitem{brunelli2009template}
Roberto Brunelli.
\newblock {\em Template matching techniques in computer vision: theory and
  practice}.
\newblock John Wiley \& Sons, 2009.

\bibitem{cai2017deep}
Zhaowei Cai, Xiaodong He, Jian Sun, and Nuno Vasconcelos.
\newblock Deep learning with low precision by half-wave gaussian quantization.
\newblock In {\em CVPR}, pages 5918--5926, 2017.

\bibitem{courbariaux2015binaryconnect}
Matthieu Courbariaux, Yoshua Bengio, and Jean-Pierre David.
\newblock Binaryconnect: Training deep neural networks with binary weights
  during propagations.
\newblock In {\em NeuriPS}, pages 3123--3131, 2015.

\bibitem{SVD}
Emily~L Denton, Wojciech Zaremba, Joan Bruna, Yann LeCun, and Rob Fergus.
\newblock Exploiting linear structure within convolutional networks for
  efficient evaluation.
\newblock In {\em NeuriPS}, 2014.

\bibitem{glorot2010understanding}
Xavier Glorot and Yoshua Bengio.
\newblock Understanding the difficulty of training deep feedforward neural
  networks.
\newblock In {\em Proceedings of the thirteenth international conference on
  artificial intelligence and statistics}, pages 249--256, 2010.

\bibitem{han2015deep}
Song Han, Huizi Mao, and William~J Dally.
\newblock Deep compression: Compressing deep neural networks with pruning,
  trained quantization and huffman coding.
\newblock {\em arXiv preprint arXiv:1510.00149}, 2015.

\bibitem{he2016deep}
Kaiming He, Xiangyu Zhang, Shaoqing Ren, and Jian Sun.
\newblock Deep residual learning for image recognition.
\newblock In {\em CVPR}, pages 770--778, 2016.

\bibitem{he2017channel}
Yihui He, Xiangyu Zhang, and Jian Sun.
\newblock Channel pruning for accelerating very deep neural networks.
\newblock In {\em ICCV}, pages 1389--1397, 2017.

\bibitem{hinton2015distilling}
Geoffrey Hinton, Oriol Vinyals, and Jeff Dean.
\newblock Distilling the knowledge in a neural network.
\newblock {\em arXiv preprint arXiv:1503.02531}, 2015.

\bibitem{howard2017mobilenets}
Andrew~G Howard, Menglong Zhu, Bo Chen, Dmitry Kalenichenko, Weijun Wang,
  Tobias Weyand, Marco Andreetto, and Hartwig Adam.
\newblock Mobilenets: Efficient convolutional neural networks for mobile vision
  applications.
\newblock {\em arXiv preprint arXiv:1704.04861}, 2017.

\bibitem{Trimming}
Hengyuan Hu, Rui Peng, Yu-Wing Tai, and Chi-Keung Tang.
\newblock Network trimming: A data-driven neuron pruning approach towards
  efficient deep architectures.
\newblock {\em arXiv preprint arXiv:1607.03250}, 2016.

\bibitem{hu2018squeeze}
Jie Hu, Li Shen, and Gang Sun.
\newblock Squeeze-and-excitation networks.
\newblock In {\em CVPR}, pages 7132--7141, 2018.

\bibitem{hubara2016binarized}
Itay Hubara, Matthieu Courbariaux, Daniel Soudry, Ran El-Yaniv, and Yoshua
  Bengio.
\newblock Binarized neural networks.
\newblock In {\em NeuriPS}, pages 4107--4115, 2016.

\bibitem{iandola2016squeezenet}
Forrest~N Iandola, Song Han, Matthew~W Moskewicz, Khalid Ashraf, William~J
  Dally, and Kurt Keutzer.
\newblock Squeezenet: Alexnet-level accuracy with 50x fewer parameters and< 0.5
  mb model size.
\newblock 2017.

\bibitem{krizhevsky2012imagenet}
Alex Krizhevsky, Ilya Sutskever, and Geoffrey~E Hinton.
\newblock Imagenet classification with deep convolutional neural networks.
\newblock In {\em NeuriPS}, pages 1097--1105, 2012.

\bibitem{lenet}
Yann LeCun, L{\'e}on Bottou, Yoshua Bengio, Patrick Haffner, et~al.
\newblock Gradient-based learning applied to document recognition.
\newblock {\em Proceedings of the IEEE}, 86(11):2278--2324, 1998.

\bibitem{long2015fully}
Jonathan Long, Evan Shelhamer, and Trevor Darrell.
\newblock Fully convolutional networks for semantic segmentation.
\newblock In {\em CVPR}, pages 3431--3440, 2015.

\bibitem{loshchilov2016sgdr}
Ilya Loshchilov and Frank Hutter.
\newblock Sgdr: Stochastic gradient descent with warm restarts.
\newblock {\em arXiv preprint arXiv:1608.03983}, 2016.

\bibitem{luo2017thinet}
Jian-Hao Luo, Jianxin Wu, and Weiyao Lin.
\newblock Thinet: A filter level pruning method for deep neural network
  compression.
\newblock In {\em ICCV}, pages 5058--5066, 2017.

\bibitem{rastegari2016xnor}
Mohammad Rastegari, Vicente Ordonez, Joseph Redmon, and Ali Farhadi.
\newblock Xnor-net: Imagenet classification using binary convolutional neural
  networks.
\newblock In {\em ECCV}, pages 525--542. Springer, 2016.

\bibitem{ren2015faster}
Shaoqing Ren, Kaiming He, Ross Girshick, and Jian Sun.
\newblock Faster r-cnn: Towards real-time object detection with region proposal
  networks.
\newblock In {\em NeuriPS}, pages 91--99, 2015.

\bibitem{rennie2003l2}
Jason Rennie.
\newblock On l2-norm regularization and the gaussian prior.
\newblock 2003.

\bibitem{romero2014fitnets}
Adriana Romero, Nicolas Ballas, Samira~Ebrahimi Kahou, Antoine Chassang, Carlo
  Gatta, and Yoshua Bengio.
\newblock Fitnets: Hints for thin deep nets.
\newblock {\em arXiv preprint arXiv:1412.6550}, 2014.

\bibitem{VGG}
Karen Simonyan and Andrew Zisserman.
\newblock Very deep convolutional networks for large-scale image recognition.
\newblock In {\em ICLR}, 2015.

\bibitem{stigler1986history}
Stephen~M Stigler.
\newblock {\em The history of statistics: The measurement of uncertainty before
  1900}.
\newblock Harvard University Press, 1986.

\bibitem{Versatile}
Yunhe Wang, Chang Xu, Chunjing Xu, Chao Xu, and Dacheng Tao.
\newblock Learning versatile filters for efficient convolutional neural
  networks.
\newblock In {\em NeuriPS}, pages 1608--1618, 2018.

\bibitem{wang2016cnnpack}
Yunhe Wang, Chang Xu, Shan You, Dacheng Tao, and Chao Xu.
\newblock Cnnpack: Packing convolutional neural networks in the frequency
  domain.
\newblock In {\em NeuriPS}, pages 253--261, 2016.

\bibitem{wen2016discriminative}
Yandong Wen, Kaipeng Zhang, Zhifeng Li, and Yu Qiao.
\newblock A discriminative feature learning approach for deep face recognition.
\newblock In {\em ECCV}, pages 499--515. Springer, 2016.

\bibitem{centerloss}
Yandong Wen, Kaipeng Zhang, Zhifeng Li, and Yu Qiao.
\newblock A discriminative feature learning approach for deep face recognition.
\newblock In {\em ECCV}, 2016.

\bibitem{wu2018shift}
Bichen Wu, Alvin Wan, Xiangyu Yue, Peter Jin, Sicheng Zhao, Noah Golmant, Amir
  Gholaminejad, Joseph Gonzalez, and Kurt Keutzer.
\newblock Shift: A zero flop, zero parameter alternative to spatial
  convolutions.
\newblock In {\em CVPR}, pages 9127--9135, 2018.

\bibitem{ResNeXt}
Saining Xie, Ross Girshick, Piotr Doll{\'a}r, Zhuowen Tu, and Kaiming He.
\newblock Aggregated residual transformations for deep neural networks.
\newblock In {\em CVPR}, pages 1492--1500, 2017.

\bibitem{yim2017gift}
Junho Yim, Donggyu Joo, Jihoon Bae, and Junmo Kim.
\newblock A gift from knowledge distillation: Fast optimization, network
  minimization and transfer learning.
\newblock In {\em CVPR}, pages 4133--4141, 2017.

\bibitem{you2017learning}
Shan You, Chang Xu, Chao Xu, and Dacheng Tao.
\newblock Learning from multiple teacher networks.
\newblock In {\em SIGKDD}, pages 1285--1294. ACM, 2017.

\bibitem{zhang2018shufflenet}
Xiangyu Zhang, Xinyu Zhou, Mengxiao Lin, and Jian Sun.
\newblock Shufflenet: An extremely efficient convolutional neural network for
  mobile devices.
\newblock In {\em CVPR}, pages 6848--6856, 2018.

\bibitem{zhong2018shift}
Huasong Zhong, Xianggen Liu, Yihui He, Yuchun Ma, and Kris Kitani.
\newblock Shift-based primitives for efficient convolutional neural networks.
\newblock {\em arXiv preprint arXiv:1809.08458}, 2018.

\bibitem{zhou2016dorefa}
Shuchang Zhou, Yuxin Wu, Zekun Ni, Xinyu Zhou, He Wen, and Yuheng Zou.
\newblock Dorefa-net: Training low bitwidth convolutional neural networks with
  low bitwidth gradients.
\newblock {\em arXiv preprint arXiv:1606.06160}, 2016.

\bibitem{zhuang2018discrimination}
Zhuangwei Zhuang, Mingkui Tan, Bohan Zhuang, Jing Liu, Yong Guo, Qingyao Wu,
  Junzhou Huang, and Jinhui Zhu.
\newblock Discrimination-aware channel pruning for deep neural networks.
\newblock In {\em NeuriPS}, pages 875--886, 2018.

\end{thebibliography}
}

\appendix

\section{Convergence of Sign and Full-precision Gradient}
AdderNets calculate the $\ell_1$ distance between the filter and the input feature, which can be formulated as 
\begin{equation}
\small
Y(m,n,t) = -\sum_{i=0}^{d}\sum_{j=0}^{d}\sum_{k=0}^{c_{in}} \vert X(m+i,n+j,k) - F(i,j,k,t)\vert.
\label{fcn:l1sup}
\end{equation}
The partial derivative of $Y$ with respect to the filters $F$ is:
\begin{equation}
\frac{\partial Y(m,n,t)}{\partial F(i,j,k,t)} = \mbox{sgn} (X(m+i,n+j,k) - F(i,j,k,t)),
\label{fcn:l1bpsup}
\end{equation}  
where $\mbox{sgn}(\cdot)$ denotes the sign function and the value of the gradient can only take +1, 0, or -1. Since Eq. (\ref{fcn:l1bpsup}) almost never takes the direction of steepest descent and the direction only gets worse as dimensionality grows, we propose to use the full-precision gradient:
\begin{equation}
\frac{\partial Y(m,n,t)}{\partial F(i,j,k,t)} =  X(m+i,n+j,k) - F(i,j,k,t).
\label{fcn:l2bpsup}
\end{equation} 

\begin{proposition}
	Denote an input patch as $x\in \mathbb{R}^n$ and a filter as $f\in \mathbb{R}^n$, the optimization problem is:
	\begin{equation}
	\mbox{arg}\min_f \vert x-f\vert.
	\label{optimsup}
	\end{equation} 
	Given a fixed learning rate $\alpha$, this problem basically cannot converge to the optimal value using sign grad (Eq. (~\ref{fcn:l1bpsup})) via gradient descent.
	\label{prop1sup}
\end{proposition}

\begin{proof}
	The optimization problem~\ref{optimsup} can be rewritten as:
	\begin{equation}
	\mbox{arg}\min_{f_1,...,f_n} \sum_{i=1}^{n}\vert x_i-f_i\vert,
	\end{equation} 
	where $x=\left\{x_1,...,x_n\right\}, f=\left\{f_1,...,f_n\right\}$. The update of $f_i$ using gradient descent is:
	\begin{equation}
	f_i^{j+1} = f_i^{j} - \alpha  \mbox{sgn}(f_i^j-x_i),
	\end{equation} 
	where $f_i^j$ denotes the $f_i$ in $j$th iteration. Without loss of generality, we assume that $f_i^0<x_i$. So we have:
	\begin{equation}
	f_i^{j+1} = f_i^{j} + \alpha = f_i^{j-1} + 2\alpha  = ... = f_i^{0}+(j+1)\alpha,
	\end{equation} 
	when $f_i^{j}<x_i$. Denote $t = \mbox{arg}\max_j f_i^j<x_i$, we have $f_i^{t+1}>=x_i$. If $f_i^{t+1}=f_i^0+(t+1)\alpha=x_i$ (\ie $\frac{(x_i-f_i^0)}{\alpha}=t+1$), $|f_i-x_i|$ can converge to the optimal value 0. However, if $f_i^{t+1}>x_i$, we have 
	\begin{equation}
	\small
	f_i^{t+2} = f_i^{t+1} - \alpha  \mbox{sgn}(f_i^{t+1}-x_i) = f_i^0 + (t+1)\alpha - \alpha = f_i^t
	\end{equation} 
	Similarly, we have $f_i^{t+3}=f_i^{t+1}$. Therefore, the inequality holds:
	\begin{equation}
	f_i^{t+2k} = f_i^{t} < x_i <  f_i^{t+2k+1} ,k\in\mathbb{N}^+
	\end{equation} 
	which demonstrate that the $f_i$ cannot converge and have an error of $x_i-f_i^t$ or $x_i-f_i^t$. The $f_i^{j}$ can converge to $x_i$ if and only if $\frac{(x_i-f_i^0)}{\alpha}\in\mathbb{Z}$, which is a strict constraint since $x_i,f_i,\alpha\in\mathbb{R}$. Moreover, the $f$ can converge to $x$ if and only if $\frac{(x_i-f_i^0)}{\alpha}\in\mathbb{Z}$ for each $f_i\in f$. The difficulty of converge increases when the number $n$ grows. In neural networks, the dimension of filters is can be very large. Therefore, problem~\ref{optimsup} basically cannot converge to its optimal value.
\end{proof}

The aim of filters is to find the most relevant part of input features, which meets the goal of Eq. (\ref{optim}). The $\alpha$ (\ie the learning rate of neural networks) can be seen as fixed when using multi-step learning rate, which is widely used in the training. According to the Proposition 1, if we use the sign gradient, the AdderNets will achieve a poor performance. 

\begin{proposition}
	For the optimization peoblem~\ref{optimsup}, $f$ can converge to the optimal value using full-precision gradient (Eq. (\ref{fcn:l2bpsup})) with a fixed learning rate $\alpha$ via gradient descent when $\alpha<1$.
	\label{prop2}
\end{proposition}
\begin{proof}
	The optimization problem~\ref{optimsup} can be rewritten as:
	\begin{equation}
	\mbox{arg}\min_{f_1,...,f_n} \sum_{i=1}^{n}\vert x_i-f_i\vert,
	\end{equation} 
	where $x=\left\{x_1,...,x_n\right\}, f=\left\{f_1,...,f_n\right\}$. The update of $f_i$ using gradient descent is:
	\begin{equation}
	f_i^{j+1} = f_i^{j} - \alpha (f_i^j-x_i),
	\end{equation} 
	where $f_i^j$ denotes the $f_i$ in $j$th iteration. If $f_i^j<x_i$, then we have the inequality:
	\begin{equation}
	f_i^{j+1} = f_i^{j} - \alpha (f_i^j-x_i) = (1-\alpha)f_i^j + \alpha x_i < x_i, 
	\end{equation} 
	and $f_i^{j+1}<f_i^j$. Without loss of generality, we assume that $f_i^0<x_i$. Then $f_i^j$ is monotone and bounded with respect to $j$, so the limit of $f_i^j$ exists and $\lim_{j \to +\infty } f_i^j\leq x_i$. Assume that $\lim_{j \to +\infty} f_i^j = l < x_i$. For $\epsilon = \alpha (x_i-l) $, there exists $k$ subject to $l-f_i^k<\epsilon$. Then we have:
	\begin{equation}
	\begin{aligned}
	f_i^{k+1}&= f_i^k+\alpha(x_i-f_i^k)\geq
	f_i^k+\alpha(x_i-l)\\&> l-\epsilon + alpha(x_i-l) = l,
	\end{aligned}
	\end{equation}
	which is a contradiction. Therefore, $\lim_{j \to +\infty} f_i^j \geq x_i$. Finally, we have $\lim_{j \to +\infty} f_i^j = x_i$, \ie $f$ can converge to the optimal value.
\end{proof}
Therefore, by utilizing the full-precision gradient, the filters can be updated precisely.

\section{Relationship Between $\ell_2$-norm and Cross-correlation}

In the main body, we propose to use a partial derivative in AdderNets, which is a clipped version of $\ell_2$-distance. Therefore, we further discuss using the $\ell_2$-distance in AdderNets instead of $\ell_1$-distance. By calculating $\ell_2$ distance between the filter and the input feature, the filters in $\ell_2$-AdderNets can be reformulated as 
\begin{equation}
\small
Y(m,n,t) = -\sum_{i=0}^{d}\sum_{j=0}^{d}\sum_{k=0}^{c_{in}} \big[ X(m+i,n+j,k) - F(i,j,k,t)\big]^2.
\label{fcn:l2}
\end{equation} 
We also use the adaptive learning rate for the $\ell_2$-AdderNets, since the magnitude of the gradient w.r.t $X$ in $\ell_2$-AdderNets would also be small. Table~\ref{tab:ImageNet} shows the classification results on the ImageNet dataset. The $\ell_2$-AdderNet can achieve almost the same accuracy with CNN. In fact, the output of the  $\ell_2$-AdderNets can be calculated as 
\begin{equation}
\scriptsize
\begin{aligned}
Y_{\ell_2}(m,n,t) =& -\sum_{i=0}^{d}\sum_{j=0}^{d}\sum_{k=0}^{c_{in}} \big[ X(m+i,n+j,k) - F(i,j,k,t)\big]^2\\
=&\sum_{i=0}^{d}\sum_{j=0}^{d}\sum_{k=0}^{c_{in}} \big[ 2  X(m+i,n+j,k)\times F(i,j,k,t)\\
&- X(m+i,n+j,k)^2 - F(i,j,k,t)^2\big]\\
=& 2Y_{CNN}(m,n,t)- \sum_{i=0}^{d}\sum_{j=0}^{d}\sum_{k=0}^{c_{in}} \big[ X(m+i,n+j,k)^2\\
& + F(i,j,k,t)^2\big].
\end{aligned}
\label{fcn:equ}
\end{equation} 
$\sum_{i=0}^{d}\sum_{j=0}^{d}\sum_{k=0}^{c_{in}}F(i,j,k,t)^2$ is same for each channel (\ie each fixed $t$). $\sum_{i=0}^{d}\sum_{j=0}^{d}\sum_{k=0}^{c_{in}} X(m+i,n+j,k)^2$ is the $\ell_2$-norm of each input patch. If this term is same for each patch, the output of $\ell_2$-AdderNet can be seen as a linear transformation of the output of CNN. Although this assumption may not always be valid, the result in Table~\ref{tab:ImageNetl2} that the performance of $\ell_2$-AdderNet and CNN are similar indicates that $\ell_2$-distance and cross-correlation have same ability to extract the information from the inputs.

\begin{table}[t]
	\centering
	\small
	\caption{Classification results on the ImageNet dataset using ResNet-18 model.}
	\begin{tabular}{|c|c|c|c|c|}
		\hline
		\textbf{Method} & \textbf{\#Mul.} & \textbf{\#Add.} & \textbf{Top-1 Acc.} & \textbf{Top-5 Acc.} \\
		\hline	
		\hline	
		$\ell_2$-AddNN & 1.8G & 3.6G &  69.6\% &  89.0\% \\
		\hline	
		$\ell_1$-AddNN & 0 & 3.6G & 66.8\% & 87.4\% \\
		\hline	
		CNN &  1.8G& 1.8G& 69.8\% & 89.1\% \\
		\hline
	\end{tabular}
	\label{tab:ImageNetl2}
	\vspace{-1.0em}
\end{table}

\end{document}